\documentclass[10pt,twocolumn,letterpaper]{article}

\usepackage[pagenumbers]{cvpr} 

\usepackage[accsupp]{axessibility}
\usepackage{graphicx}
\usepackage{amsmath}
\usepackage{amssymb}
\usepackage{booktabs}
\usepackage{mathtools}
\usepackage{amsthm}
\usepackage{soul}
\usepackage{multirow}
\newtheorem{theorem}{Theorem}
\newtheorem{lemma}{Lemma}
\DeclareMathOperator*{\argmax}{arg\,max}

\usepackage[pagebackref,breaklinks,colorlinks]{hyperref}

\usepackage[capitalize]{cleveref}
\crefname{section}{Sec.}{Secs.}
\Crefname{section}{Section}{Sections}
\Crefname{table}{Table}{Tables}
\crefname{table}{Tab.}{Tabs.}

\def\cvprPaperID{4} 
\def\confName{CVPR}
\def\confYear{2023}

\begin{document}

\title{Certified Adversarial Robustness Within Multiple Perturbation Bounds}


\author{Soumalya Nandi \quad Sravanti Addepalli \thanks{Equal Contribution}  \quad Harsh Rangwani \footnotemark[1] \quad R. Venkatesh Babu \\
Vision and AI Lab, Indian Institute of Science, Bengaluru 
}

\maketitle

\begin{abstract}
Randomized smoothing (RS) is a well known certified defense against adversarial attacks, which creates a smoothed classifier by predicting the most likely class under random noise perturbations of inputs during inference. While initial work focused on robustness to $\ell_2$ norm perturbations using noise sampled from a Gaussian distribution, subsequent works have shown that different noise distributions can result in robustness to other $\ell_p$ norm bounds as well. In general, a specific noise distribution is optimal for defending against a given $\ell_p$ norm based attack. In this work, we aim to improve the certified adversarial robustness against multiple perturbation bounds simultaneously. Towards this, we firstly present a novel \textit{certification scheme}, that effectively combines the certificates obtained using different noise distributions to obtain optimal results against multiple perturbation bounds. We further propose a novel \textit{training noise distribution} along with a \textit{regularized training scheme} to improve the certification within both $\ell_1$ and $\ell_2$ perturbation norms simultaneously. Contrary to prior works, we compare the certified robustness of different training algorithms across the same natural (clean) accuracy, rather than across fixed noise levels used for training and certification. We also empirically invalidate the argument that training and certifying the classifier with the same amount of noise gives the best results. The proposed approach achieves improvements on the ACR (Average Certified Radius) metric across both $\ell_1$ and $\ell_2$ perturbation bounds. Code available at \url{https://github.com/val-iisc/NU-Certified-Robustness}
\end{abstract}

\vspace{-0.3cm}
\section{Introduction}
\label{Introduction}

Deep neural networks are vulnerable to carefully crafted input perturbations known as adversarial attacks \cite{2013,szegedy2014intriguing}. These perturbations are imperceptible to human eyes, but are sufficient to fool the classifier into making wrong predictions. In order to improve the robustness of models against such attacks, one line of research is on \emph{empirical defenses}, which augments the training data with adversarial attacks during training \cite{goodfellow2015explaining,madry2018towards,zhang2019theoretically,wu2020adversarial,wong2020fast}. Another line of research is on \emph{certified defenses}, which gives mathematically proven probabilistic guarantees within guarded areas around an input $x$, such that any perturbation within that region fails to deceive the network \cite{raghunathan2020certified, wong2018provable,pmlr-v97-cohen19c}. \emph{Randomized smoothing} (RS) \cite{pmlr-v97-cohen19c} is one such certified defense that makes a network certifiably robust within an $\ell_2$ norm ball by creating a smoothed version of a base classifier. This is achieved by using zero mean additive Gaussian noise augmentations of the inputs during both training and inference. 

In this work, we firstly highlight the following shortcomings observed in the current literature of randomized smoothing, and further propose solutions to address the same: $1)$ Existing defenses focus mostly on achieving certified robustness within a given threat model \cite{pmlr-v97-cohen19c,yang2020randomized}. However, for real-life applications, it is necessary to improve the robustness of models against more than one $\ell_p$ norm threat models simultaneously \cite{tramer2019adversarial,croce2020provable}. $2)$ Prior works \cite{pmlr-v97-cohen19c} recommend using the same noise magnitude during both training and certification for optimal performance, but we empirically demonstrate that this does not hold true in practice (Table \ref{table:1}). $3)$ It has been a standard practice to compare various defenses against a common magnitude of noise that is used for training/ certification \cite{pmlr-v97-cohen19c,https://doi.org/10.48550/arxiv.2006.04062}. However, we show that this may not be the best way of comparing defenses, since the use of different training/ certification schemes may result in different robustness-accuracy trade-offs, making it hard to compare defenses from an end user perspective. For example, an improved robustness at a lower clean accuracy may not necessarily mean improved performance overall. 
Following are our contributions to address the above issues:
\begin{itemize}
\vspace{-0.1cm}
\item We propose a novel certification scheme, that combines the benefits of \emph{Gaussian} and \emph{Uniform} noise based smoothed classifiers to create a stronger hybrid smoothed classifier in terms of both $\ell_1$ and $\ell_2$ norm robustness guarantees.
\vspace{-0.1cm}
\item Contrary to existing works, we show that using the same noise during both training and inference does not always lead to optimal robustness. Further, we propose to compare defenses by targeting a fixed value of clean accuracy, rather than by fixing the noise used for training/ inference which, can be misleading. 
\vspace{-0.1cm}
\item We propose the use of \emph{Normal-Uniform} distribution as the training noise distribution, alongside a regularizer that enforces similarity between the training and inference noise distributions. Using the proposed approach, we demonstrate improved certified robustness against both $\ell_2$ and $\ell_1$ attacks simultaneously. 


\end{itemize}

\section{Background}
\label{chapter:3}

\subsection{Randomized Smoothing}

A base classifier $f$ is first trained using cross-entropy loss on images with Gaussian noise based augmentations, where noise sampled from $\mathcal{N}(0,\sigma^2I)$ is added to every image. During certification, $f$ is transformed to a smoothed classifier $g$ such that for a given test image $x$, $g$ outputs the most probable class across different noise augmentations of $x$, using noise sampled from the same distribution as that used for training. Since it is not feasible to compute this probability exactly for Neural Networks, it is estimated using a random sample of $n$ noise vectors generated from the smoothing distribution for each test image $x$. The augmented images are passed through the trained base classifier to obtain $n$ predictions for $x$. Let $\hat{c}_A$ be the class predicted the most number of times and $n_A$ be the number of times $\hat{c}_A$ has been predicted, then the estimated probability is $p_A = n_A/n$. Next a one-sided $(1-\alpha)$ lower confidence bound of $p_A$ is calculated as $\underline{p_A}$. If $\underline{p_A} > 0.5$, then the $\ell_2$ norm robust radius is returned as $\sigma\Phi^{-1}(\underline{p_A})$, otherwise the sample is said to be \emph{ABSTAINED}. Cohen \etal \cite{pmlr-v97-cohen19c} use $n=100,000$ and $\alpha = 0.001$, resulting in a $0.1\%$ chance of returning a falsely certified data point. More details on the computation of certified radius for $\ell_1$ and $\ell_2$ threat models are presented in section \ref{suppl:s1.1} of the Supplementary.

\subsection{Evaluation metrics for Certified Robustness}
\label{sec:3.3}
We use the following metrics for evaluation: \emph{Clean accuracy} and \emph{robustness}. Clean accuracy refers to the accuracy on natural unperturbed images. For robustness, we use a metric called \emph{Average Certified Radius} or ACR, introduced by Zhai \etal \cite{zhai2020macer}. ACR is the average value of certified radii of all the considered test data points as shown below:
\begin{equation}
\label{eqn:6}
    ACR := \frac{1}{|D_{test}|} \sum_{(x,y) \in D_{test}} CR(f,x).\mathbf{1}_{g(x)=y}
\end{equation} where $D_{test}$ is the considered test dataset, $CR(f,x)$ is the certified radius of $x$ for base classifier $f$, and $\mathbf{1}_{g(x)=y}$ is the indicator function for correct classification. As the noise level $\sigma$ is increased, some data points will have higher certified radii, but the number of wrongly classified data points would also increase. For misclassified data points we would have, $CR(f,x).\mathbf{1}_{g(x)=y} = 0$. So increasing the noise magnitude $\sigma$, does not necessarily increase the ACR. Hence in addition to robustness, ACR also captures the robustness-accuracy trade-off of a model.

\section{Related Works}
\label{Related}


\subsection{Empirical Defenses}
Empirical Defenses are heuristic based approaches that achieve adversarial robustness without specific mathematical guarantees. Their adversarial robustness is thus evaluated against strong empirical attacks such as AutoAttack \cite{croce2020reliable} and GAMA \cite{sriramanan2020gama}. Adversarial training \cite{goodfellow2015explaining,madry2018towards,wong2020fast} is one such defense, which maximizes a classification loss to generate adversarial attacks and minimizes the loss on the attacked images for training. Although empirical defenses achieve the best possible robustness today \cite{zhang2019theoretically,wu2020adversarial}, they do not provide robustness guarantees, which may be crucial for security critical applications such as autonomous driving. In the history of empirical defenses, several early methods which used gradient obfuscation as a defense strategy \cite{buckman2018thermometer, xie2018mitigating, song2018pixeldefend, guo2018countering} have later been broken later by stronger attacks \cite{athalye2018obfuscated}, highlighting the need for robustness guarantees.

\subsection{Certified Defenses}
Certified defenses provide provable probabilistic or exact guarantees \cite{raghunathan2020certified,pmlr-v97-cohen19c,yang2020randomized,wong2018provable,sinha2018certifiable} to ensure that for any input $x$, the classifier's output is consistent within a defined neighbourhood around $x$, i.e., any attack within this radius is unsuccessful to deceive the network. Certified defenses can be categorized as, i) Exact methods \cite{https://doi.org/10.48550/arxiv.1702.01135,https://doi.org/10.48550/arxiv.1610.06940,https://doi.org/10.48550/arxiv.1709.10207,https://doi.org/10.48550/arxiv.1706.07351,https://doi.org/10.48550/arxiv.1709.09130,https://doi.org/10.48550/arxiv.1705.01040}, ii) Convex Optimization based methods \cite{wong2018provable,https://doi.org/10.48550/arxiv.1810.12715,https://doi.org/10.48550/arxiv.1811.02625,https://doi.org/10.48550/arxiv.1810.07481,https://doi.org/10.48550/arxiv.1811.00866,pmlr-v80-mirman18b}, and iii) Randomized Smoothing \cite{pmlr-v97-cohen19c,yang2020randomized} based methods. The former two categories are highly computationally expensive and architecture dependent, thus not scalable to large networks. Being architecture independent, randomized smoothing has an unparalleled advantage over the other two categories in this regard. Cohen \etal \cite{pmlr-v97-cohen19c} provide a tight $\ell_2$ certified robustness guarantee and Yang \etal \cite{yang2020randomized} provide the same for an $\ell_1$ adversary along with the theoretical foundation.

\subsection{Robustness to Multiple Threat Models} 

Most existing defenses aim to achieve robustness against only a single type of adversary, such as attacks constrained within a given $\ell_p$ norm bound. Tramer \etal \cite{tramer2019adversarial} propose an empirical defense against more than one $\ell_p$ norm attacks simultaneously, by performing adversarial training on all or the worst attack amongst all considered threat models for the given sample. Croce and Hein \cite{https://doi.org/10.48550/arxiv.2105.12508} propose a defense that utilizes the geometrical relation between different $\ell_p$ norm balls to obtain robustness against a union of $\ell_p$ norm threat models. The authors also introduce a similar framework \cite{croce2020provable} for provable robustness against multiple $\ell_p$ norm perturbations. However, to the best of our knowledge, certified robustness guarantee for more than one $\ell_p$ norm bound using randomized smoothing is unexplored.

\subsection{Randomized Smoothing (RS) based methods} 

Cohen \etal \cite{pmlr-v97-cohen19c} recommend training and certifying with the same noise augmentations to achieve optimal performance. Yang \etal \cite{yang2020randomized} also use the same strategy of training and certifying with the same noise magnitude. Few recent works \cite{zhai2020macer,https://doi.org/10.48550/arxiv.2006.04062} propose the use of regularized robust training methods to improve the certified radius of the smoothed classifier. Zhai \etal \cite{zhai2020macer} propose MACER, which uses a robustness loss in addition to the standard training framework provided by Cohen \etal \cite{pmlr-v97-cohen19c}, to maximize the robust radius. Jeong and Shin \cite{https://doi.org/10.48550/arxiv.2006.04062} propose a regularizer that controls the prediction consistency over noise to increase the certified robustness of smoothed classifiers, while also controlling the accuracy-robustness trade-off. Motivated by this, we propose a regularizer that is better suited to the proposed defense of using Normal-Uniform noise during training, and a combination of Normal and Uniform noise during certification. Although these works differ in their training methods, they also perform training and certification with the same noise augmentations and compare results across fixed levels of noise magnitudes. We highlight the limitations of the same (Table \ref{table:1}) and propose a novel comparison perspective across different methods.

\section{Proposed Approach}

We now discuss the various contributions of this work during training, certification and inference in greater detail. 

\subsection{Proposed Certification Method for robustness within multiple threat models}
\label{sec:certification}

As noted in prior works \cite{pmlr-v97-cohen19c,yang2020randomized}, the optimal smoothing distribution for $\ell_1$ norm robustness is Uniform distribution and that for $\ell_2$ norm robustness is Gaussian distribution. In order to achieve robustness within both $\ell_1$ and $\ell_2$ norm bounds, we perform certification twice using noise sampled from Uniform and Gaussian distributions respectively. While the certification process for a single threat model is straightforward, combining the certificates obtained for different smoothing distributions effectively is non-trivial. Towards this, we propose a novel strategy to effectively combine the predictions and certifications obtained from the individual smoothing distributions. A Uniform (Gaussian) smoothed classifier gives $4$ outcomes for every test image:  confidence in the predicted label (Abstained or not), predicted label $y_{U}$ ($y_{G}$), $\ell_1$ norm certified radius $r_{U}^{l_1}$ ($r_{G}^{l_1}$) and $\ell_2$ norm certified radius $r_{U}^{l_2}$ ($r_{G}^{l_2}$). We build a hybrid smoothed classifier from these $8$ outputs using the rules presented in Table \ref{table:5} depending on the different possibilities arising from the predictions and their confidence. Table \ref{table:6} shows that the proposed certification method improves the $\ell_1$ ACR when compared to the use of Gaussian smoothing alone for certification, without diminishing the $\ell_2$ ACR. The clean accuracy also improves using this approach.
\begin{table}
\begin{center}
\caption{Proposed certification method that combines the certificates of the two smoothed classifiers to build a stronger hybrid classifier with improved robustness in both $\ell_1$ and $\ell_2$ norm bounds.}
\vspace{-0.3cm}
\label{table:5}
\renewcommand{\arraystretch}{1.2}
\begin{sc}
\resizebox{1.0\linewidth}{!}{%
\begin{tabular}{lccc}
\hline
Category & $\ell_1$ radius & $\ell_2$ radius & predicted label\\
\hline
$y_{U}=y_{G}$ & $Max(r_{U}^{l_1},r_{G}^{l_1})$ & $Max(r_{U}^{l_2},r_{G}^{l_2})$ & Common label\\

Only $y_{U}$ is abstained  & $r_{G}^{l_1}$ & $r_{G}^{l_2}$ & $y_{G}$\\

Only $y_{G}$ is abstained  & $r_{U}^{l_1}$ & $r_{U}^{l_2}$ & $y_{U}$\\

Both are abstained  & 0 & 0 & abstained\\

$y_{U} \neq y_{G}$ & 0 & 0 & abstained\\
\hline
\end{tabular}}
\end{sc}
\end{center}
\vspace{-0.3cm}
\end{table}

\begin{table*}
\vskip 0.15in
\begin{center}
\caption{The proposed certification strategy can significantly increase the $\ell_1$ ACR without compromising on the $\ell_2$ ACR.}
\label{table:6}
\renewcommand{\arraystretch}{1.2}
\begin{sc}
\resizebox{1.0\linewidth}{!}{%
\begin{tabular}{lcccc}
\hline
Training & Certification & ~~~~Clean Acc~~~~ & ~~~~~$\ell_2$ ACR~~~~ & ~~~~$\ell_1$ ACR~~~~\\
\hline
$NU(\sigma_N=0.50,\sigma_U=0.433)$+$R_S(\beta=3)$ & $Gaussian(\sigma=0.60)$ & 59.30 & 0.742 & 0.742\\
$NU(\sigma_N=0.50,\sigma_U=0.433)$+$R_S(\beta=3)$ & $Gaussian(\sigma=0.60)+Uniform(\sigma=0.65)$ & 60.26 & 0.742 & \textbf{0.783}\\
\hline
$NU(\sigma_N=1.00,\sigma_U=0.866)$+$R_S(\beta=2)$ & $Gaussian(\sigma=1.00)$ & 44.49 & 0.769 & 0.769\\
$NU(\sigma_N=1.00,\sigma_U=0.866)$+$R_S(\beta=2)$ & ~~~~~~~$Gaussian(\sigma=1.00)+Uniform(\sigma=1.16)$~~~~~~~ & 45.13 & 0.769 & \textbf{0.823}\\
\hline
\end{tabular}}
\end{sc}
\end{center}
\vspace{-0.6cm}
\end{table*}

\subsection{Noise Magnitude for training and inference}
\label{sec:4.2}
We empirically show that the use of same noise magnitude ($\sigma$) for both training and inference does not always result in the best overall performance. As shown in Table \ref{table:1}, both clean accuracy and ACR improve when a lower $\sigma$ is used during certification. When a consistency regularizer is used, a higher $\sigma$ at test time improves robust accuracy with a marginal drop in clean accuracy. We thus select the best combination of train and test $\sigma$, rather than using the same value for both. 

\begin{table}
\begin{center}
\caption{The use of different noise magnitudes ($\sigma$) during training and testing may result in better overall performance as shown on full testset of CIFAR-10 for Gaussian smoothing \cite{pmlr-v97-cohen19c} with and without consistency regularization \cite{https://doi.org/10.48550/arxiv.2006.04062}. }
\vspace{-0.3cm}
\label{table:1}
\renewcommand{\arraystretch}{1.2}
\begin{sc}
\resizebox{1.0\linewidth}{!}{%
\begin{tabular}{lcccc}
\toprule
Method & Train $\sigma$ & Test $\sigma$ & Clean-Acc & $\ell_2$ ACR \\
\midrule
Gaussian  & 1.00 & 1.00 & 45.90 & 0.492\\
Gaussian  & 1.00 & 0.80 & 47.00 & \textbf{0.531}\\
\hline
Gaussian + Consistency Reg. & 0.25 & 0.25 & 74.41 & 0.545\\
Gaussian + Consistency Reg. & 0.25 & 0.30 & 72.81 & \textbf{0.584}\\
\bottomrule
\end{tabular}}
\end{sc}
\end{center}
\vspace{-0.7cm}
\end{table}

\subsection{Normal-Uniform Noise Distribution}
In this work, we aim to achieve improved adversarial robustness to $\ell_2$ norm perturbations using Gaussian smoothing and to $\ell_l$ norm perturbations using Uniform smoothing simultaneously. Towards this, we propose to use a combination of both distributions during training, which we refer to as the \emph{Normal-Uniform} distribution as shown in Figure \ref{fig:1}. We consider two independent random variables, $X \sim N(0,\sigma_N^2)$, and $Y \sim U(-\sqrt{3}\sigma_U,\sqrt{3}\sigma_U)$. Then, $Z=X+Y$ is defined to follow a \emph{Normal-Uniform} distribution with parameters ($\sigma_N,\sigma_U$). 
\begin{figure}
\begin{center}
\centerline{\includegraphics[width=0.8\columnwidth]{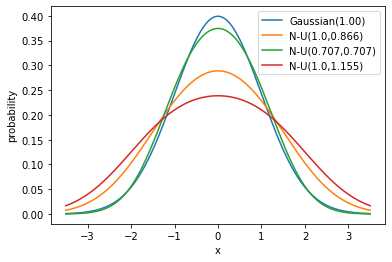}}
\vspace{-0.4cm}
\caption{Shape of the proposed Normal-Uniform distribution for different kurtosis values.} 
\label{fig:1}
\end{center}
\vspace{-1cm}
\end{figure}

The use of Normal-Uniform noise for training allows the model to be trained on noise sampled from both Gaussian and Uniform distributions simultaneously, while also allowing the choice of parameters such that $Z$ has a negative kurtosis. The \textbf{kurtosis} of a probability distribution is the measure of its shape in terms of its tailedness. A distribution with negative kurtosis (called \emph{platykurtic distribution}) has thinner tails than that of a Gaussian distribution, i.e., the distribution is lesser prone to generating extreme values when compared to the Gaussian distribution. The use of a platykurtic distribution during training reduces the probability of generating outliers, thereby improving the training stability as shown in Table \ref{table:3}, where ACR is higher using the proposed distribution. The last two rows in each block use the same noise magnitude, showing that the negative kurtosis improves performance. More details on \textbf{kurtosis} are provided in section \ref{suppl:s1.3} of the Supplementary.


\begin{table}
\vskip 0.15in
\begin{center}
\caption{Robustness to perturbations on CIFAR-10: Negative kurtosis K(X) coupled with a higher training $\sigma$ can have significantly better ACR. The last two rows in each block use the same noise magnitude, showing that negative kurtosis improves performance.}
\vspace{-0.3cm}
\label{table:3}
\renewcommand{\arraystretch}{1.2}
\begin{sc}
\resizebox{1.0\linewidth}{!}{%
\begin{tabular}{lccccc}
\hline
Training Distribution(X) & Training $\sigma$ & K(X) & Test $\sigma$ & Clean-Acc & $\ell_2$ ACR \\
\hline
Gaussian ($\sigma=0.25$) \cite{pmlr-v97-cohen19c} & 0.250 & 0.00 & \multirow{3}{3em}{0.25} & 77.64 & 0.429\\
Gaussian ($\sigma=0.331$) & 0.331 & 0.00 & & 67.95 & 0.391\\
NU ($\sigma_N=0.25, \sigma_U=0.217$) & 0.331 & -0.22 & & 75.20 & \textbf{0.450}\\
\hline
Gaussian ($\sigma=0.50$) \cite{pmlr-v97-cohen19c} & 0.500 & 0.00 & \multirow{3}{3em}{0.50} & 64.19 & 0.508\\
Gaussian ($\sigma=0.661$) & 0.661 & 0.00 & & 56.89 & 0.526\\
NU ($\sigma_N=0.50, \sigma_U=0.433$) & 0.661 & -0.22 & & 61.26 & \textbf{0.560}\\
\hline
Gaussian ($\sigma=1.00$) \cite{pmlr-v97-cohen19c} & 1.000 & 0.00 & \multirow{3}{3em}{1.00} & 45.90 & 0.492\\
Gaussian($\sigma=1.323$) & 1.323 & 0.00 & & 37.71 & 0.477\\
NU ($\sigma_N=1.00, \sigma_U=0.866$) & 1.323 & -0.22 & & 45.70 & \textbf{0.592}\\
\hline
\end{tabular}}
\end{sc}
\end{center}
\vspace{-0.6cm}
\end{table}
\subsection{Similarity Regularizer}
Consistency regularizers are commonly used in the literature of adversarial defenses \cite{addepalli2020bpfc,zhang2019theoretically,sriramanan2020gama,sriramanan2021nuat}. In empirical defenses, Zhang \etal \cite{https://doi.org/10.48550/arxiv.1901.08573} introduce TRADES regularizer that enforces similarity between the outputs of a clean image and the corresponding perturbed image. Li \etal \cite{li2019certified} introduce \emph{stability training} for certified robustness in the same spirit as TRADES, where a Gaussian noise augmented image is used instead of the adversarially perturbed image. Zhai 
\etal \cite{zhai2020macer} introduce MACER, that uses a regularizer to maximize the certified radius for $\ell_2$ norm perturbations. The current state-of-the-art method \cite{https://doi.org/10.48550/arxiv.2006.04062} also uses a consistency regularizer to enforce consistency in predictions over noisy samples. In this work, we introduce the following \emph{similarity regularizer} which is based on KL-divergence,
\begin{multline}
\label{eqn:2}
    R_{S} = \mathrm{KL}(f(x+\mathbf{NU})||f(x+\mathbf{N})) + \\ \mathrm{KL}(f(x+\mathbf{NU})||f(x+\mathbf{U}))
\end{multline}
where, $f$ is the base classifier that outputs the $K$ dimensional probability vector, $\mathbf{NU}$, $\mathbf{N}$ and $\mathbf{U}$ are noise vectors sampled from the Normal-Uniform$(\sigma_N,\sigma_U)$, Normal$(0,\sigma_N^2)$ and Uniform$(-\sqrt{3}\sigma_U,\sqrt{3}\sigma_U)$ distributions respectively. $\sigma_N$ and $\sigma_U$ are the standard deviations of Gaussian and Uniform sub-parts of the Normal-Uniform distribution respectively.

Although we propose to use the Normal-Uniform noise distribution during training, we obtain certification using noise from the individual Normal and Uniform distributions as recommended by Cohen \etal \cite{pmlr-v97-cohen19c} and Yang \etal \cite{yang2020randomized} respectively. The certificates obtained are further combined using the process discussed in Section-\ref{sec:certification}. Since the noise distributions used during training and certification are different, the proposed regularizer (Eq.\ref{eqn:2}) helps in aligning predictions of the Normal-Uniform corrupted images with that of each distribution individually, thereby aligning the training and certification stages.
The overall loss function used in the proposed method is shown below, where $\mathcal{L}_{\mathrm{CE}}(.)$ is the cross entropy loss and $\beta$ is a hyperparamater.
\begin{equation}
\label{eqn:3}
    L := \mathcal{L}_{\mathrm{CE}}(f(x+\mathbf{NU}),y) + \beta \cdot R_{S}
\end{equation}

As shown in Table \ref{table:4}, on using the proposed regularizer $R_S$, we observe a significant boost in the ACR against $\ell_2$ norm perturbations at slightly lower clean accuracy, when Gaussian noise based smoothing is used for certification. 

\begin{table}
\begin{center}
\caption{Effect of the proposed \emph{similarity regularizer} $R_S$ when Normal-Uniform noise is used during training, along with Gaussian noise during certification on CIFAR-10.}
\vspace{-0.3cm}
\label{table:4}
\renewcommand{\arraystretch}{1.2}
\begin{sc}
\resizebox{\linewidth}{!}{%
\begin{tabular}{lccc}
\hline
Training & Certification & Clean Acc & $\ell_2$ ACR\\
\hline
$NU(\sigma_N=0.50,\sigma_U=0.433)$ & \multirow{2}{*}{$\mathcal{N}(0.60^2)$} & 61.77 & 0.570\\
$NU(\sigma_N=0.50,\sigma_U=0.433)$+$R_S(\beta=3)$ &  & 59.30 & \textbf{0.742}\\
\hline
$NU(\sigma_N=1.00,\sigma_U=0.866)$ & \multirow{2}{*}{$\mathcal{N}(1.00^2)$} & 45.70 & 0.592\\
$NU(\sigma_N=1.00,\sigma_U=0.866)$+$R_S(\beta=2)$ & & 44.49 & \textbf{0.769}\\
\hline
\end{tabular}}
\end{sc}
\end{center}
\vspace{-0.6cm}
\end{table}

\subsection{Comparison Strategy across Methods}
\label{sec:4.5}
A standard practice of comparing robust classifiers across different training strategies is to use a common noise magnitude ($\sigma$) during inference. However, the use of different training strategies can result in vastly different clean and robust performance metrics for the same noise magnitude used during inference, making it very difficult to compare a method with higher clean accuracy and lower ACR against a method with lower clean accuracy and better ACR. Given that variation of training and inference noise magnitude does give the flexibility of trading off robust accuracy with clean accuracy, from an end user perspective, it is important to know which method gives the best ACR for a specified value of clean accuracy. We therefore propose to firstly tune the training and certification $\sigma$ to achieve the required level of clean accuracy with optimal ACR, and further compare the ACR across different methods.

\section{Experiments and Results}
\subsection{Experimental Setup}
We present an empirical evaluation of the proposed approach on the full test set of CIFAR-10 \cite{krizhevsky2009learning} unless specified otherwise. As discussed in Section-\ref{sec:3.3}, we evaluate the performance of models on their clean accuracy and Average Certified radius or ACR. To reproduce results from the baselines, we use the codes officially released by the authors. We use the model architecture as ResNet-110 \cite{https://doi.org/10.48550/arxiv.1512.03385}, as used in the baselines. Our proposed model is trained for $300$ epochs with a batch size of $400$. We use cosine learning rate schedule with SGD optimizer. In this work, we refer to the Gaussian Smoothing baseline proposed by Cohen \etal \cite{pmlr-v97-cohen19c} as \emph{Gaussian} and the Consistency Regularization baseline by Jeong \etal \cite{https://doi.org/10.48550/arxiv.2006.04062} as \emph{Gaussian+Consistency} or \emph{Gaussian+cons}.
\subsection{Comparing Different Smoothing Measures}
\label{sec:5.2}
As discussed in Section-\ref{sec:4.5}, comparison across different approaches was done against a fixed value of smoothing $\sigma$ in prior works \cite{zhai2020macer,https://doi.org/10.48550/arxiv.2006.04062}. In Section-\ref{sec:4.5}, we motivate the need for comparing against a fixed value of clean accuracy instead. We now present another scenario where comparing across a fixed noise $\sigma$ can give misleading results. 

We consider a comparison between methods that use different smoothing distributions during inference. For example, Yang \etal \cite{yang2020randomized} theoretically show that Uniform distribution is optimal for robustness within $\ell_1$ norm bound. However, it is not straightforward to show this in practice by comparing two methods that use different noise distributions for certification, as shown in Table \ref{table:2}. Merely comparing two methods against the same inference noise magnitude using the $\ell_1$ ACR metric gives an impression that Gaussian smoothing is a better approach. 
However, one cannot conclude which smoothing is better overall because they create two different types of robustness-accuracy trade-offs as shown in the left image of Figure \ref{fig:2}. There are approximately $40\%$ data-points which have an $\ell_1$ certified radius greater than or equal to $0.5$ under Gaussian smoothing, but with uniform smoothing all the data-points have certified radii less than $0.5$. On the other hand, the latter gives more correctly classified data-points than the former, making it hard to interpret which approach is better. Therefore, comparing two differently smoothed classifiers by fixing the smoothing noise level apriori fails to capture the true robustness-accuracy trade-off. The plot on the right in Figure \ref{fig:2}, and results in Table \ref{table:7} clearly show that smoothing with Uniform noise gives better $\ell_1$ robustness, which rightly reflects the theoretical results by Yang \etal \cite{yang2020randomized}. The proposed comparison method also works effectively using the proposed approach where we combine certifications from two different smoothing strategies during inference, and hence the noise level of the hybrid classifier cannot be determined easily. 

\begin{table}
\begin{center}
\caption{Comparing $\ell_1$ certified robustness using noise from Gaussian and Uniform  distributions for training and certification: While in theory, Uniform noise is known to be better, merely comparing ACR using a fixed noise magnitude for certification gives a false impression that Gaussian smoothing is optimal.}
\vspace{-0.3cm}
\label{table:2}
\renewcommand{\arraystretch}{1.2}
\begin{sc}
\resizebox{1.0\linewidth}{!}{%
\begin{tabular}{cccc}
\toprule
$\sigma$ & ~~~~Model~~~~ & ~~~~Clean-Acc~~~~ & ~~~~$\ell_1$ ACR~~~~ \\
\midrule
0.25  & Gaussian Smoothing & 76.40 & 0.430\\
0.25 & Uniform Smoothing & 86.40 & 0.331\\
\bottomrule
\end{tabular}}
\end{sc}
\end{center}
\vspace{-0.6cm}
\end{table}

\begin{figure}
\begin{center}
\centerline{\includegraphics[width=\columnwidth]{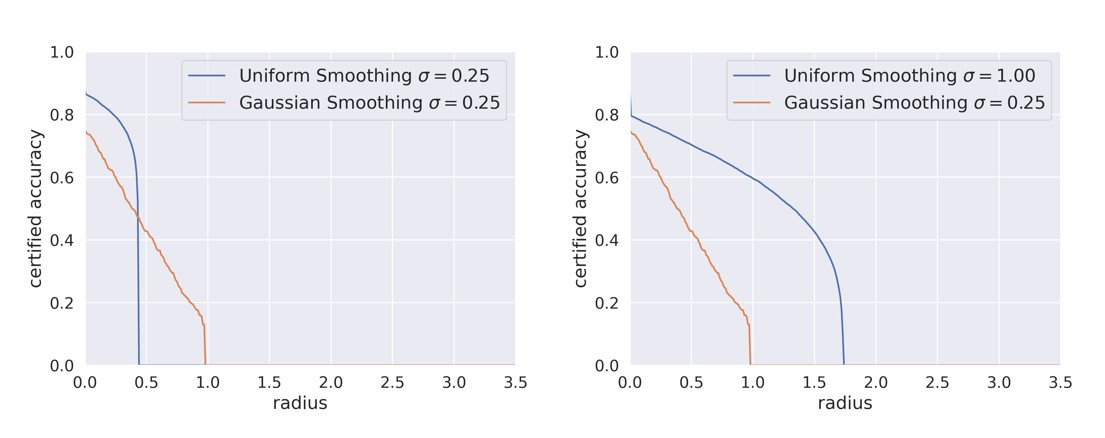}}
\end{center}
\vspace{-0.9cm}
\caption{Comparison of $\ell_1$ norm robustness-accuracy trade-off using Gaussian and Uniform smoothing for (\textbf{Left}) a fixed noise magnitude ($\sigma$) of $0.25$. (\textbf{Right}) a fixed clean accuracy of $75\%$.} 
\label{fig:2}
\vspace{-0.1cm}
\end{figure}

\begin{table}
\begin{center}
\caption{Comparison of Gaussian and Uniform Smoothing against fixed clean accuracy clearly shows that Uniform distribution is better for $\ell_1$ certified robustness as shown by Yang \etal \cite{yang2020randomized}.}
\label{table:7}
\renewcommand{\arraystretch}{1.2}
\vspace{-0.2cm}
\begin{sc}
\resizebox{1.0\linewidth}{!}{%
\begin{tabular}{cccc}
\hline
$\sigma$ & Model & ~~~~Clean-Acc~~~~~ & ~~~~$\ell_1$ ACR ~~~~\\
\hline
0.25 & ~~~~Gaussian Smoothing~~~~ & 76.40 & 0.430\\
1.00  & Uniform Smoothing & 75.20 & \textbf{0.967}\\
\hline
\end{tabular}}
\end{sc}
\end{center}
\vspace{-0.6cm}
\end{table}

\begin{table*}
\begin{center}
\caption{Performance of the proposed approach when compared to baselines across different fixed levels of clean accuracy (in each block), in terms of $\ell_1$, $\ell_2$ and average ACR on  the full test set of CIFAR-10.}
\vspace{-0.2cm}
\label{table:8}
\renewcommand{\arraystretch}{1.15}
\begin{sc}
\resizebox{1.0\linewidth}{!}{%
\begin{tabular}{cccccc}
\hline
Training scheme & Certification & Clean Acc & $\ell_1$ ACR & $\ell_2$ ACR & avg ACR\\
\hline
$Gaussian(\sigma=1.00)$\cite{pmlr-v97-cohen19c} & $Gaussian(\sigma=1.00)$ & 45.90 & 0.492 & 0.492 & 0.492\\
$Gaussian(\sigma=1.00)$+cons\cite{https://doi.org/10.48550/arxiv.2006.04062} & $Gaussian(\sigma=1.00)$ & 45.96 & 0.762 & 0.762 & 0.762\\
$NU(\sigma_N=1.00,\sigma_U=0.866)$+$R_S(\beta=2)$(Ours) & $Gaussian(\sigma=1.00) + Unif(\sigma=1.16)$(Ours) & 45.13 & \textbf{0.823} & \textbf{0.769} & \textbf{0.796}\\
\hline
$Gaussian(\sigma=0.84)$\cite{pmlr-v97-cohen19c} & $Gaussian(\sigma=0.84)$ & 50.62 & 0.513 & 0.513 & 0.513\\
$Gaussian(\sigma=0.81)$+cons\cite{https://doi.org/10.48550/arxiv.2006.04062} & $Gaussian(\sigma=0.81)$ & 50.39 & 0.762 & 0.762 & 0.762\\
$NU(\sigma_N=0.75,\sigma_U=0.65)$+$R_S(\beta=3)$(Ours) & $Gaussian(\sigma=0.75) + Unif(\sigma=0.90)$(Ours) & 51.07 & \textbf{0.824} & \textbf{0.768} & \textbf{0.796}\\
\hline
$Gaussian(\sigma=0.73)$\cite{pmlr-v97-cohen19c} & $Gaussian(\sigma=0.73)$ & 55.98 & 0.521 & 0.521 & 0.521\\
$Gaussian(\sigma=0.68)$+cons\cite{https://doi.org/10.48550/arxiv.2006.04062} & $Gaussian(\sigma=0.68)$ & 55.52 & 0.761 & \textbf{0.761} & 0.761\\
$NU(\sigma_N=0.60,\sigma_U=0.52)$+$R_S(\beta=4)$(Ours) & $Gaussian(\sigma=0.60) + Unif(\sigma=0.70)$(Ours) & 55.80 & \textbf{0.790} & 0.762 & \textbf{0.771}\\
\hline
$Gaussian(\sigma=0.62)$\cite{pmlr-v97-cohen19c} & $Gaussian(\sigma=0.62)$ & 60.34 & 0.527 & 0.527 & 0.527\\
$Gaussian(\sigma=0.57)$+cons\cite{https://doi.org/10.48550/arxiv.2006.04062} & $Gaussian(\sigma=0.57)$ & 60.47 & 0.734 & 0.734 & 0.734\\
$NU(\sigma_N=0.50,\sigma_U=0.433)$+$R_S(\beta=3)$(Ours) & $Gaussian(\sigma=0.60) + Unif(\sigma=0.65)$(Ours) & 60.26 & \textbf{0.783} & \textbf{0.742} & \textbf{0.762}\\
\hline
$Gaussian(\sigma=0.50)$\cite{pmlr-v97-cohen19c} & $Gaussian(\sigma=0.50)$ & 64.19 & 0.508 & 0.508 & 0.508\\
$Gaussian(\sigma=0.48)$+cons\cite{https://doi.org/10.48550/arxiv.2006.04062} & $Gaussian(\sigma=0.48)$ & 64.91 & 0.707 & 0.707 & 0.707\\
$NU(\sigma_N=0.40,\sigma_U=0.346)$+$R_S(\beta=3)$(Ours) & $Gaussian(\sigma=0.50) + Unif(\sigma=0.50)$(Ours) & 65.27 & \textbf{0.731} & \textbf{0.708} & \textbf{0.720}\\
\hline
$Gaussian(\sigma=0.39)$\cite{pmlr-v97-cohen19c} & $Gaussian(\sigma=0.39)$ & 70.37 & 0.503 & 0.503 & 0.503\\
$Gaussian(\sigma=0.37)$+cons\cite{https://doi.org/10.48550/arxiv.2006.04062} & $Gaussian(\sigma=0.37)$ & 70.78 & 0.648 & \textbf{0.648} & 0.648\\
$NU(\sigma_N=0.30,\sigma_U=0.260)$+$R_S(\beta=2)$(Ours) & $Gaussian(\sigma=0.40) + Unif(\sigma=0.35)$(Ours) & 70.91 & \textbf{0.666} & 0.637 & \textbf{0.652}\\
\hline
$Gaussian(\sigma=0.30)$\cite{pmlr-v97-cohen19c} & $Gaussian(\sigma=0.30)$ & 74.26 & 0.455 & 0.455 & 0.455\\
$Gaussian(\sigma=0.25)$+cons\cite{https://doi.org/10.48550/arxiv.2006.04062} & $Gaussian(\sigma=0.25)$ & 74.41 & 0.545 & 0.545 & 0.545\\
$NU(\sigma_N=0.20,\sigma_U=0.17)$+$R_S(\beta=4)$(Ours) & $Gaussian(\sigma=0.30) + Unif(\sigma=0.25)$(Ours) & 74.52 & \textbf{0.577} & \textbf{0.556} & \textbf{0.566}\\
\hline
\end{tabular}}
\end{sc}
\end{center}
\vspace{-0.5cm}
\end{table*}

\vspace{0.1cm}
\subsection{SOTA Comparison Results}
The results presented so far have focused on highlighting the individual contributions of our work. In this section, we present the overall results using all the recommendations presented in this paper. During training we use the proposed Normal-Uniform (NU) distribution along with the proposed similarity regularizer (Equation \ref{eqn:2}). During certification we create two smoothed classifiers, one with Uniform smoothing and the other with Gaussian smoothing, and further combine them using the proposed certification method presented in Section \ref{sec:certification}. As discussed in Section \ref{sec:4.5}, we compare across different methods against fixed levels of clean accuracy set to the following - $45\%, 50\%, 55\%, 60\%, 65\%, 70\%$ and $75\%$. As shown in Table \ref{table:8}, the proposed approach has significantly higher ACR in all cases when compared to the baselines. The training parameters $\sigma_N$ and $\sigma_U$ are chosen to set the kurtosis of the Normal-Uniform distribution to a negative value ($-0.22$), which gives the relation between $\sigma_N$ and $\sigma_U$ as $\sigma_U = \sqrt{3}\sigma_N/2$. We further tune the hyper-parameter ($\sigma_N$) to achieve the required level of clean accuracy. Experimentally, a low value of $\beta \in \{2,3,4\}$ for our proposed regularization (Equation \ref{eqn:3}) gives the best performance for the considered levels of clean accuracy. We further note from Table \ref{table:8} that for Gaussian smoothing \cite{pmlr-v97-cohen19c}, reducing the clean accuracy below $60\%$ by increasing the noise level results in a drop in ACR, which does not happen in the proposed method. 

In supplementary material, we present a detailed ablation study of the proposed method. We present the impact of variation in the hyperparameter $\beta$, impact of choice of regularizer and the effect of Kurtosis on the final performance of the certified classifier in terms of its clean accuracy, $\ell_1$ ACR and $\ell_2$ ACR. 

\vspace{-0.5cm}
\section{Conclusion}
In this work, we propose several aspects related to the training, certification and inference strategies for improving certified adversarial robustness within multiple perturbation bounds. Contrary to prior belief, we show that training and inference noise levels need not be the same to achieve optimal results. Further, we propose to compare different methods against a fixed clean accuracy rather that using a fixed noise level during inference. Thirdly, we propose an effective way of combining the certifications obtained using two different noise distributions. We next present the proposed training methodology of using noise sampled from the Normal-Uniform distribution during training, and ensuring similar representations during inference by using a consistency regularizer w.r.t. both Gaussian and Uniform distributions that are used during inference. We motivate the need for each of the individual strategies using several experimental results, and also present the impact of combining all discussed strategies to achieve state-of-the-art results in robustness against a combination of $\ell_1$ and $\ell_2$ perturbation bounds. We hope our work motivates further research on certified robustness within multiple perturbation bounds, which is very important in a real world setting.

\vspace{-0.2cm}
\section{Acknowledgments}
\vspace{-0.2cm}
This work was supported by the research grant CRG/2021/005925 from SERB, DST, Govt. of India. Sravanti is supported by Google PhD Fellowship, and Harsh is supported by Prime Minister's Research Fellowship.

{\small
\bibliographystyle{ieee_fullname}
\bibliography{main}
}

\clearpage

\twocolumn[
  \begin{@twocolumnfalse}
\begin{center}
\textbf{\Large Supplementary material}
\vspace{1cm}
\end{center}
 \end{@twocolumnfalse}
  ]

\setcounter{equation}{0}
\setcounter{figure}{0}
\setcounter{table}{0}
\setcounter{section}{0}
\makeatletter
\renewcommand{\theequation}{S\arabic{equation}}
\renewcommand{\thefigure}{S\arabic{figure}}
\renewcommand{\thetable}{S\arabic{table}}
\renewcommand{\thesection}{S\arabic{section}}

\section{Background}
\label{suppl:s1}

\subsection{Randomized Smoothing}
\label{suppl:s1.1}
Let $f: \mathbb{R}^d \rightarrow \{1,2,3,...,K\}$ be a neural network that maps a $d$ dimensional image to one of the $K$ classes. Using Randomized Smoothing, the base classifier $f(x)$ can be transformed to a smoothed classifier $g(x)$ that has inherent probabilistic certified guarantees. Given an input $x$, the smoothed classifier $g(x)$ outputs the most likely class as predicted by the base classifier, across different augmentations of the input image, as shown below:
\begin{equation}
\label{eqn:s1}
    g(x)=\argmax_c P[f(x+\varepsilon)=c]
\end{equation}
Here, $\varepsilon$ is generated from a smoothing measure $\mu$. Considering $\mu$ to be isotropic Gaussian, Cohen \etal \cite{pmlr-v97-cohen19c} show that $g(x)$ inherits certified robustness in $\ell_2$ norm through the following theorem.
\begin{theorem}
\label{theo:1}
(Restating theorem $1$ by Cohen \etal \cite{pmlr-v97-cohen19c}): Let $\varepsilon \sim N(0,\sigma^2I)$. Suppose $c_A \in \{1,2,3,...,K\}$ and $\underline{p_A},\overline{p_B} \in [0,1]$ satisfy:
    $P(f(x+\varepsilon)=c_A)\geq\underline{p_A}\geq\overline{p_B}\geq max_{c \neq c_A}P(f(x+\varepsilon)=c)$. Then $g(x+\delta)=c_A$ for all $||\delta||_2 < R$, where $R = \frac{\sigma}{2}(\Phi^{-1}(\underline{p_A})-\Phi^{-1}(\overline{p_B}))$, $\Phi^{-1}$ being the inverse of standard Gaussian CDF.
\end{theorem}

Yang \etal \cite{yang2020randomized} show using the following theorem that by considering $\mu$ as a Uniform distribution, $g(x)$ has provable robustness guarantee against $\ell_1$ norm constrained attacks.
\begin{theorem}
\label{theo:2}
(Restating theorem $I.8$ by Yang \etal \cite{yang2020randomized}):
Suppose $H$ is a smoothed classifier smoothed by the uniform distribution on the cube $[-\lambda,\lambda]^d$, such that $H(x) = (H(x)_1,...,H(x)_C)$ is a vector of probabilities that $H$ assigns to each class $1,...,C$. If $H$ correctly predicts the class $y$ on input $x$, and the probability of the correct class is $\rho \stackrel{def}{=} H(x)_y>1/2$, then $H$ continues to predict the correct class when $x$ is perturbed by any $\eta$ with $||\eta||_1 < 2\lambda(\rho-0.5)$.
\end{theorem}
\subsection{Consistency Regularization}
\label{suppl:s1.2}
Jeong and Shin \cite{https://doi.org/10.48550/arxiv.2006.04062} attempts to achieve better generalization performance of the base classifier over noise augmentation. Since, during certification the model is evaluated on noise augmented inputs, it is logical to use noise augmented inputs for training also, but the variance of the noise distribution hampers the stability of the training process. This work introduces a regularizer on top of the standard cross entropy loss that controls the prediction consistency over noisy samples. The overall loss function is,
\begin{multline}
\label{eqn:s2}
    L:= \frac{1}{m}\sum_{i}(\mathcal{L}_{CE}(F(x+\delta_i),y) + \\
    \lambda \cdot KL(\hat{F}(x) || F(x+\delta_i)) + \eta \cdot H(\hat{F}(x))
\end{multline}
where $KL( \cdot || \cdot)$ is the KL-divergence term and $H(\cdot)$ is the entropy term. $F(x)$ is the differentiable function on which the classifier $f$ is built, $\delta$ is Gaussian noise and $\hat{F}(x) = \frac{1}{m}\sum_{i}F(x+\delta_i)$. $\lambda$ and $\eta$ are hyperparameters. $\mathcal{L}_{CE}$ is the cross entropy function. The KL term reduces the variance of the predictions while the entropy term prevents the variance to become $0$. In the paper, $m$ is fixed as $2$ and $\eta$ is fixed at $0.5$. The certification process is exactly same as that of \emph{Gaussian Smoothing}.
\vspace{0.2cm}
\subsection{Kurtosis of a distribution}
\label{suppl:s1.3}
Let $X$ be a real valued random variable, then the kurtosis, $K(X)$ of the probability distribution of $X$ is defined as,
\begin{equation}
\label{eqn:s3}
    K(X) = \frac{\mu_4}{\mu_2^2} - 3
\end{equation}
where $\mu_i$ is the $i^{th}$ ordered central moment of $X$, i.e.,
\begin{equation}
\label{eqn:s4}
    \mu_i = E[X - E[X]]^i
\end{equation}
Kurtosis measures the shape of a distribution in terms of its tailedness. If a probability distribution has fat tails, that is the random variable $X$ has a good amount of area under the curve on its tails then for points in that region, $X-E[X]$ would be large in magnitude. So, $[X-E[X]]^4$ would produce even larger positive values. Therefore a high value of $E[X-E[X]]^4$ or $\mu_4$ denotes a fat tailed distribution. The very similar argument follows to conclude that a random variable with low value of $\mu_4$ has a thin tailed probability distribution. In general a standardized metric is used hence, $\frac{\mu_4}{\mu_2^2}$. The above metric $K(X)$ is used to measure the tailedness of a distribution relative to Gaussian distribution. For a normal distribution, $\frac{\mu_4}{\mu_2^2} = 3$, so $K(X)$ becomes $0$ and the distribution is called a \emph{mesokurtic} distribution. A distribution with a negative $K(X)$ is called a \emph{platykurtic} distribution and has thinner tails than that of a Gaussian distribution, whereas a distribution with a positive $K(X)$ is called a \emph{leptokurtic} distribution which has fatter tails than a Gaussian distribution. A platykurtic distribution is less prone to generate outliers than a Gaussian distribution while a leptokurtic distribution produces outliers with a higher probability than that of a Gaussian distribution. For example, Uniform distribution is a platykurtic distribution with $K(X) = -1.2$. It does not generate outliers at all, whereas a standard Laplace distribution has $K(X) = 3$ which generates much more outliers than Gaussian distribution.

\vspace{0.1cm}
\section{Theoretical properties of Normal-Uniform distribution}
\label{suppl:s2}
\vspace{0.1cm}

Let $X \sim N(0,\sigma_N^2)$ and $Y \sim U(-\lambda,\lambda)$ independently, then $Z=X+Y$ will follow a \emph{Normal-Uniform} distribution with parameters ($\sigma_N,\lambda$). The $\sigma_N$ denotes the standard deviation of Normal distribution. For a $U(-\lambda,\lambda)$ distribution, the standard deviation is $\sigma_U = \frac{\lambda}{\sqrt{3}}$. Therefore to define both the distributions with the same parameter, we have used $\sigma_U$ instead of $\lambda$ in all our experiments. Hence the \emph{Normal-Uniform} distribution is defined with the parameters ($\sigma_N,\sigma_U$). By controlling these parameters, we can effectively control the shape of the distribution (Figure \ref{fig:s3}) and can make it to behave more like a Gaussian distribution or a Uniform distribution accordingly (Figure \ref{fig:s1}).
\begin{lemma}
\label{lemma:4.1}
If $X \sim \mathcal{N}(0,\sigma^2)$ and $Y \sim \mathcal{U}(-\lambda,\lambda)$ with $X$ and $Y$ being independent, then $Z=X+Y$ has the pdf as $f_Z(z) = \frac{1}{2\lambda}[\Phi(\frac{z+\lambda}{\sigma})-\Phi(\frac{z-\lambda}{\sigma})]$ with $z \in (-\infty,\infty)$, where $\Phi(.)$ is the CDF of standard normal distribution.
\end{lemma}
\begin{proof}
pdf of $X$ is $f_X(x) = \frac{1}{\sigma\sqrt{2\pi}}e^{-\frac{x^2}{2\sigma^2}}$ with $x \in \mathbb{R}$. Pdf of $Y$ is $f_Y(y) = \frac{1}{2\lambda}$ with $y \in (-\lambda,\lambda)$. Then,
$F_Z(t) = P[Z \leq t]$
\begin{center}
    $=\iint_{x+y \leq t}f_X(x) f_Y(y) \,dx\,dy$ \hspace{3mm}
$=\int_{-\infty}^{\infty}f_Y(y)\{\int_{-\infty}^{t-y}f_X(x)dx\}dy$
     \hspace{3mm}
     $=\int_{-\infty}^{\infty}f_Y(y)\Phi(\frac{t-y}{\sigma})dy$
\end{center}
So, the pdf is
\begin{center}
    $f_Z(z) = \frac{1}{\sigma}\int_{-\infty}^{\infty}f_Y(y)\phi(\frac{t-y}{\sigma})dy$\hspace{3mm}
$=\frac{1}{\sigma}\int_{-\lambda}^{\lambda}\frac{1}{2\lambda}\phi(\frac{t-y}{\sigma})dy = {\frac{1}{2\lambda}[\Phi(\frac{z+\lambda}{\sigma})-\Phi(\frac{z-\lambda}{\sigma})]} $
\end{center}
with $z \in (-\infty,\infty)$.
\end{proof}

\begin{figure}[t]
\vskip 0.2in
\begin{center}
\centerline{\includegraphics[width=\columnwidth]{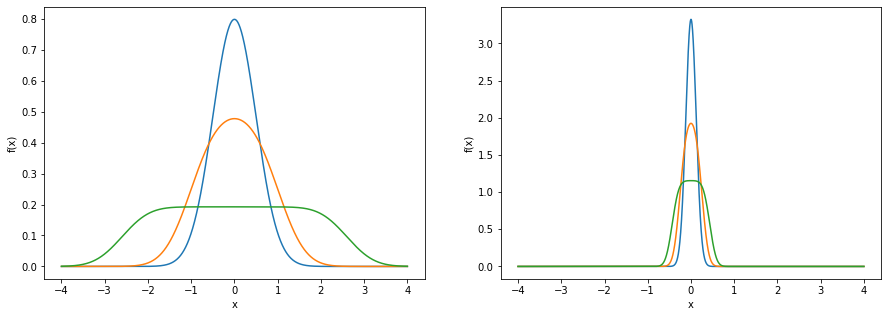}}
\caption{By controlling $\sigma_{N}$ and $\sigma_{U}$, the shape of Normal-Uniform probability distribution function (pdf) can be adjusted from bell shaped to flat surfaced. \textbf{Left}: Comparison of pdf of Normal$(\sigma_{N}=0.5)$(Blue), Normal-Uniform$(\sigma_{N}=0.5,\sigma_{U}=0.577)$(Orange), Normal-Uniform$(\sigma_{N}=0.5,\sigma_{U}=1.500)$(Green) \textbf{Right}: Comparison of pdf of Normal$(\sigma_{N}=0.12)$(Blue), Normal-Uniform$(0,\sigma_{N}=0.12,\sigma_{U}=0.144)$(Orange), Normal-Uniform$(\sigma_{N}=0.12,\sigma_{U}=0.25)$(Green)}
\label{fig:s1}
\end{center}
\vskip -0.2in
\end{figure}

\begin{lemma}
\label{lemma:4.4}
If $X \sim N(0,\sigma^2)$ and $Y \sim U(-\lambda,\lambda)$ with $X$ and $Y$ being independent, then $Z=X+Y$ has the cdf as $F_Z(t) = \frac{\sigma}{2\lambda}[\frac{t+\lambda}{\sigma}\Phi(\frac{t+\lambda}{\sigma})+\phi(\frac{t+\lambda}{\sigma})-\frac{t-\lambda}{\sigma}\Phi(\frac{t-\lambda}{\sigma})-\phi(\frac{t-\lambda}{\sigma})]$ with $t \in (-\infty,\infty)$, where $\Phi(.)$ and $\phi(.)$ are the CDF and pdf of standard normal distribution respectively.
\end{lemma}

\begin{proof}
$F_Z(t) = P[X \leq t] = \int_{-\infty}^{t}\frac{1}{2\lambda}[\Phi(\frac{z+\lambda}{\sigma})-\Phi(\frac{z-\lambda}{\sigma})]dx$\\
\\
Let $I=$
    $\int_{-\infty}^{t}\Phi(\frac{x+\lambda}{\sigma})dx$
    $=\sigma\int_{-\infty}^{\frac{t+\lambda}{\sigma}}\Phi(z)dz$
    $=\sigma[\Phi(z)z-\int\phi(z)zdz]_{-\infty}^{\frac{t+\lambda}{\sigma}}$
    $=\sigma[\Phi(z)z + \phi(z)]_{-\infty}^{\frac{t+\lambda}{\sigma}}$\\
    \\
    $=\sigma[\Phi(\frac{t+\lambda}{\sigma})(\frac{t+\lambda}{\sigma})+\phi(\frac{t+\lambda}{\sigma})]$\\
Similarly calculating for the second term and replacing in the original equation, we get\\
$F_Z(t) = \frac{\sigma}{2\lambda}[\frac{t+\lambda}{\sigma}\Phi(\frac{t+\lambda}{\sigma})+\phi(\frac{t+\lambda}{\sigma})-\frac{t-\lambda}{\sigma}\Phi(\frac{t-\lambda}{\sigma})-\phi(\frac{t-\lambda}{\sigma})]$
\end{proof}

\begin{figure}[t]
\vskip 0.2in
\begin{center}
\centerline{\includegraphics[width=\columnwidth]{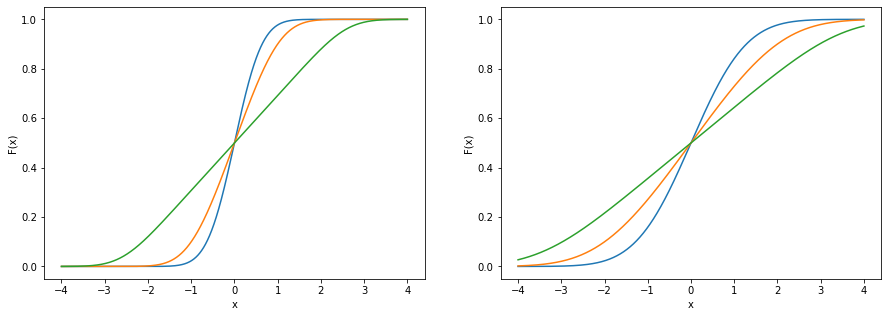}}
\caption{Similar to figure \ref{fig:s1}, by controlling the $\sigma_{N}$ and $\sigma_{U}$, the shape of Normal-Uniform cdf can be adjusted from S shaped to a straightline. \textbf{Left}: Comparison of cdfs of Normal$(\sigma_{N}=0.5)$(Blue), Normal-Uniform$(\sigma_{N}=0.5,\sigma_{U}=0.577)$(Orange), Normal-Uniform$(\sigma_{N}=0.5,\sigma_{U}=1.500)$(Green) \textbf{Right}: Comparison of cdfs of Normal$(\sigma_{N}=1.0)$(Blue), Normal-Uniform$(\sigma_{N}=1.0,\sigma_{U}=1.155)$(Orange), Normal-Uniform$(\sigma_{N}=1.0,\sigma_{U}=2.00)$(Green)}
\label{fig:s2}
\end{center}
\vskip -0.2in
\end{figure}

\begin{lemma}
\label{lemma:4.5}
If $X \sim N(0,\sigma^2)$ and $Y \sim U(-\lambda,\lambda)$ with $X$ and $Y$ being independent, then $Z=X+Y$ has $K(Z) = \frac{3\sigma^4 + 2\sigma^2\lambda^2 + (\lambda^4/5)}{(\sigma^2+(\lambda^2/2))^2} - 3$
\end{lemma}

\begin{proof}
We have $E(Z) = 0$ and $Var(Z) = Var(X) + Var(Y) = \sigma^2 + \frac{\lambda^2}{3}$\\
Now, $\mu_4 = E(Z-E(Z))^4 = E(Z)^4 = E[X^4+4X^3Y+6X^2Y^2+4XY^3+Y^4]$\\
$=E(X^4)+6E(X^2)E(Y^2)+E(Y^4) = 3\sigma^4 +6\sigma^2\frac{\lambda^2}{3} + E(Y^4)$\\
\\
We have $E(Y^4) = \int_{-\lambda}^{\lambda}\frac{y^4}{2\lambda}dy = \frac{y^5}{10\lambda} |_{-\lambda}^{\lambda} = \frac{\lambda^4}{5} $\\
\\
so $K(Z) = \frac{\mu_4}{\mu_2^2} - 3 = \frac{3\sigma^4 + 2\sigma^2\lambda^2 + (\lambda^4/5)}{(\sigma^2+(\lambda^2/2))^2} - 3$
\end{proof}
We can vary the shape of the distribution from a bell shaped curve to flat surfaced curve by controlling the $\sigma$ and $\lambda$ appropriately as shown in figure \ref{fig:s3}.

\vspace{0.2cm}
\section{Ablation study}
\label{suppl:s3}

\vspace{0.3cm}

\subsection{Effect of tuning parameter $\beta$}
\label{suppl:s3.1}
We investigate the effect of the regularizer tuning parameter $\beta$. As usual, when we increase $\beta$, initially the ACRs increase and clean accuracy decreases. A prominent robustness-accuracy trade-off is visible in table \ref{table:9}. However, as we further increase $\beta$, the clean accuracy oscillates between $55\%$ and $52\%$, while the $\ell_1$ ACR gets stagnant around $0.770$ and $\ell_2$ ACR around $0.750$ before both dropping drastically.
\begin{table}[htbp!]
\vskip 0.15in
\begin{center}
\caption{Effect of $\beta$ when trained on Normal-Uniform($\sigma_N=0.50, \sigma_U=0.433$) and certified on our proposed method on a subset of $500$ test images of CIFAR10 with Gaussian smoothing($\sigma=0.60$) + Uniform smoothing($\sigma=0.65$).}
\label{table:9}
\begin{sc}
\begin{tabular}{cccc}
\hline
$\beta$ & Clean Acc & $\ell_1$ ACR & $\ell_2$ ACR\\
\hline
0 & 62.00 & 0.601 & 0.570\\
2 & 60.40 & 0.743 & 0.714\\
3 & 59.00 & 0.776 & 0.734\\
4 & 58.20 & \textbf{0.779} & 0.749\\
6 & 55.20 & 0.773 & \textbf{0.751}\\
8 & 52.60 & 0.771 & 0.749\\
10 & 55.20 & 0.777 & 0.741\\
12 & 55.40 & 0.766 & 0.750\\
14 & 54.60 & 0.770 & 0.749\\
16 & 52.00 & 0.758 & 0.737\\
18 & 50.80 & 0.769 & 0.753\\
20 & 51.60 & 0.765 & 0.751\\
24 & 9.400 & 0.215 & 0.215\\
\hline
\end{tabular}
\end{sc}
\end{center}
\vskip -0.1in
\end{table}
\vspace{0.3cm}
\subsection{Effect of choice of KL term}
\label{suppl:s3.2}
Our proposed regularizer has two KL terms, each associated with one smoothed classifier. This results in $3$ forward passes for each batch of images during training as we need $3$ different outputs $F(x+NU), F(x+U), F(x+N)$ to calculate the regularizer. In this section we try out the following regularizers having only one KL term. 
\begin{equation}
\label{eqn:4}
    R_N = KL(F(x+NU)||F(x+N))
\end{equation}
\begin{equation}
\label{eqn:5}
    R_U = KL(F(x+NU)||F(x+U))
\end{equation}
Table \ref{table:10} shows the results under few setups. In most cases, proposed \emph{similarity} regularizer dominates the robustness for comparable clean accuracy. The models trained with $R_U$ regularizer provide better clean accuracy than $R_S$ and also comparable $\ell_1$ ACR. The $\ell_2$ ACR drops though as can be seen in $3^{rd}$ and $4^{th}$ rows of table \ref{table:10}. On the other hand, regularizer $R_N$ gives a better $\ell_2$ ACR as compared to $R_U$ at the cost of slight decrease in clean accuracy. The use of $R_U$, forces the training noise to behave more like a Uniform distribution, whereas $R_N$ makes it more similar with Gaussian distribution. Table \ref{table:12} shows that training with Gaussian noise and certifying with Uniform noise performs better than doing the opposite. That is why, using $R_U$ results in higher decrease in $\ell_2$ ACR than that of $\ell_1$ ACR while using $R_N$. Our proposed \emph{similarity} regularizer performs better in terms of both the ACRs at comparable clean accuracy as compared to both $R_N$ and $R_U$. However, if one has to choose between $R_N$ and $R_U$ only, $R_N$ is more preferable.
\begin{table*}[htbp!]
\vskip 0.15in
\begin{center}
\caption{Effect of using a single KL term instead of two KL terms in the regularizer on a subset of $500$ test images of CIFAR10.}
\label{table:10}
\begin{sc}
\resizebox{1.0\linewidth}{!}{%
\begin{tabular}{ccccc}
\hline
Training & Certification & Clean Acc & $\ell_1$ ACR & $\ell_2$ ACR\\
\hline
\\
$NU(\sigma_N=0.50,\sigma_U=0.433)$+$R_N(\beta=6)$ & \multirow{3}{*}{$Gaussian(\sigma=0.60)+Unif(\sigma=0.650)$} & 60.00 & 0.770 & 0.746\\
$NU(\sigma_N=0.50,\sigma_U=0.433)$+$R_U(\beta=6)$ & & 58.40 & 0.778 & 0.710\\
$NU(\sigma_N=0.50,\sigma_U=0.433)$+$\mathbf{R_S}(\beta=3)$ & & 59.00 & 0.776 & 0.734\\
\\
\hline
\\
$NU(\sigma_N=1.00,\sigma_U=0.866)$+$R_N(\beta=6)$ & \multirow{3}{*}{$Gaussian(\sigma=1.00)+Unif(\sigma=1.160)$} & 42.80 & 0.806 & 0.763\\
$NU(\sigma_N=1.00,\sigma_U=0.866)$+$R_U(\beta=6)$ & & 45.00 & 0.828 & 0.742\\
$NU(\sigma_N=1.00,\sigma_U=0.866)$+$\mathbf{R_S}(\beta=4)$ & & 44.00 & 0.858 & 0.789\\
\\
\hline
\end{tabular}}
\end{sc}
\end{center}
\vskip -0.1in
\end{table*}
\vspace{0.3cm}
\subsection{Effect of kurtosis}
\label{suppl:s3.3}
So far we have used a kurtosis value of $-0.22$ for all our proposed training experiments. As we know a Gaussian distribution has kurtosis value of $0$ and a Uniform distribution has kurtosis value of $-1.2$, so when we introduce too much negative kurtosis in our training distribution, then it deviates from a Gaussian distribution and behaves more like a Uniform distribution (Figure \ref{fig:s2}). In the next subsection, we have evidence that training with Gaussian noise and then certifying with Uniform noise works better than training with Uniform noise and certifying with Gaussian noise. So keeping a mild negative kurtosis results in better robustness guarantees. Table \ref{table:11} shows that decreasing the kurtosis further does not help in terms of the robustness-accuracy trade-off. The choice of $\sigma_N$ and $\sigma_U$ may seem arbitrary, but a specific relation selects them. For first row, $\sigma_U = \frac{1.5}{\sqrt{3}}\sigma_N$, and for the last row, $\sigma_U = \frac{2}{\sqrt{3}}\sigma_N$.

\begin{table}
\vskip 0.15in
\begin{center}
\caption{Effect of amount of negative kurtosis $K(X)$ in the training distribution when certified on our proposed method with Gaussian smoothing($\sigma=1.00$) + Uniform smoothing($\sigma=1.160$) when tested on a subset of $500$ images of CIFAR10.}
\label{table:11}
\begin{sc}
\resizebox{1.0\linewidth}{!}{%
\begin{tabular}{cccccc}
\hline
Training & $K(X)$ & Clean Acc & $\ell_1$ ACR & $\ell_2$ ACR & Avg ACR\\
\hline
$NU(\sigma_N=1.00,\sigma_U=0.866)$ & -0.22 & 45.00 & 0.671 & 0.625 & \textbf{0.648}\\
$NU(\sigma_N=1.00,\sigma_U=1.155)$ & -0.39 & 38.40 & 0.682 & 0.606 & 0.644\\
\hline
\end{tabular}}
\end{sc}
\end{center}
\vskip -0.1in
\end{table}

\begin{figure}
\vskip 0.2in
\begin{center}
\centerline{\includegraphics[width=\columnwidth]{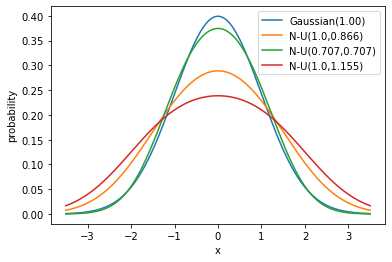}}
\caption{Shape of Normal-Uniform distribution for different kurtosis values.} 
\label{fig:s3}
\end{center}
\vskip -0.2in
\end{figure}

\subsection{Performance of the baselines}
\label{suppl:s3.4}
In this section we evaluate the performance of the baselines on our proposed certification method. As per the suggestions provided by the considered baselines \cite{pmlr-v97-cohen19c,https://doi.org/10.48550/arxiv.2006.04062}, the noise level $\sigma$ is fixed at different levels in prior and then the model is trained with proposed methods from the considered baselines at that fixed noise level. For certification we create a Gaussian smoothed classifier and a Uniform smoothed classifier each with the same fixed noise level and combine the certificates as per our proposed method. Table \ref{table:12} describes the results for $\sigma \in \{0.25, 0.50\}$ and shows that the performance is inferior to our proposed training noise distribution. We get an insignificant improvement in $\ell_1$ ACR under Gaussian noise training when the model is certified with our proposed hybrid smoothed classifier over Gaussian smoothed classifier. This shows that the proposed use of Normal-Uniform noise distribution as the training noise plays a key role in order to create highly robust hybrid smoothed classifier. Another notable finding is that performance of the models trained with Gaussian noise augmentation and certified under Uniform smoothing are far better than the performance of the models trained with Uniform noise augmentation and certified under Gaussian Smoothing. 

\begin{table}
\vskip 0.15in
\begin{center}
\caption{Performance of the baselines under proposed certification method when tested on a sample of $500$ test images of CIFAR10.}
\label{table:12}
\begin{sc}
\resizebox{1.0\linewidth}{!}{%
\begin{tabular}{cccccc}
\hline
$\sigma$ & Training & Certification & Clean Acc & $\ell_1$ ACR & $\ell_2$ ACR\\
\hline
\multirow{11}{*}{0.25} & \multirow{3}{*}{Gaussian} & Gaussian & 76.40 & 0.430 & 0.430\\
 & & Uniform & 76.60 & 0.276 & 0.005\\
 & & Ours & \textbf{76.60} & \textbf{0.438} & \textbf{0.430}\\
 \\
 & \multirow{3}{*}{Gaussian+Consistency} & Gaussian & 73.40 & 0.535 & 0.535\\
 & & Uniform & 73.20 & 0.291 & 0.005\\
 & & Ours & \textbf{73.60} & \textbf{0.538} & \textbf{0.535}\\
 \\
 & \multirow{3}{*}{Uniform} & Gaussian & 44.60 & 0.180 & 0.180\\
 & & Uniform & 86.40 & 0.331 & 0.006\\
 & & Ours & 58.80 & 0.279 & 0.178\\
\hline
\multirow{11}{*}{0.50} & \multirow{3}{*}{Gaussian} & Gaussian & 64.80 & 0.523 & 0.523\\
 & & Uniform & 65.20 & 0.406 & 0.007\\
 & & Ours & \textbf{65.40} & \textbf{0.543} & \textbf{0.523}\\
 \\
 & \multirow{3}{*}{Gaussian+Consistency} & Gaussian & 64.60 & 0.702 & 0.702\\
 & & Uniform & 64.80 & 0.450 & 0.008\\
 & & Ours & \textbf{65.00} & \textbf{0.720} & \textbf{0.702}\\
 \\
 & \multirow{3}{*}{Uniform} & Gaussian & 14.80 & 0.062 & 0.062\\
 & & Uniform & 79.40 & 0.568 & 0.01\\
 & & Ours & 36.60 & 0.271 & 0.062\\
\hline
\end{tabular}}
\end{sc}
\end{center}
\vskip -0.1in
\end{table}

\end{document}